\documentclass[sigconf]{acmart}

\usepackage{booktabs} %

\usepackage{amssymb}
\usepackage{amsfonts}
\usepackage{amsmath}
\usepackage{url}
\usepackage{hyperref} %
\usepackage[nameinlink,capitalize]{cleveref}
\usepackage[caption=false,font=normalsize,labelfont=sf,textfont=sf]{subfig}
\newcommand{\norm}[1]{\left\Vert#1\right\Vert}
\newcommand{\trace}{\mathrm{tr}}
\newcommand{\Ebb}{\mathbb{E}}

\usepackage{balance} %

\begin{document}
\title{Disturbance Grassmann Kernels for Subspace-Based Learning}

\author{Junyuan Hong}
\affiliation{%
  \institution{University of Science and Technology of China}
  \city{Hefei}
  \state{China}
  \postcode{230027}
}
\email{jyhong@mail.ustc.edu.cn}

\author{Huanhuan Chen}
\authornote{Corresponding author}
\affiliation{%
  \institution{University of Science and Technology of China}
  \city{Hefei}
  \state{China}
  \postcode{230027}
}
\email{hchen@ustc.edu.cn}

\author{Feng Lin}
\affiliation{%
  \institution{University of Science and Technology of China}
  \city{Hefei}
  \state{China}
  \postcode{230027}
}
\email{lf1995@mail.ustc.edu.cn}

\renewcommand{\shortauthors}{J. Hong et al.}

\begin{abstract}
In this paper, we focus on subspace-based learning problems, where data elements are linear subspaces instead of vectors.
To handle this kind of data, Grassmann kernels were proposed to measure the space structure and used with classifiers, e.g., Support Vector Machines (SVMs).
However, the existing discriminative algorithms mostly ignore the instability of subspaces, which would cause the classifiers to be misled by disturbed instances.
Thus we propose considering all potential disturbances of subspaces in learning processes to obtain more robust classifiers.
Firstly, we derive the dual optimization of linear classifiers with disturbances subject to a known distribution, resulting in a new kernel, Disturbance Grassmann (DG) kernel.
Secondly, we research into two kinds of disturbance, relevant to the subspace matrix and singular values of bases, with which we extend the Projection kernel on Grassmann manifolds to two new kernels.
Experiments on action data indicate that the proposed kernels perform better compared to state-of-the-art subspace-based methods, even in a worse environment.
\end{abstract}

\copyrightyear{2018} 
\acmYear{2018} 
\setcopyright{acmlicensed}
\acmConference[KDD '18]{The 24th ACM SIGKDD International Conference on Knowledge Discovery \& Data Mining}{August 19--23, 2018}{London, United Kingdom}
\acmBooktitle{KDD '18: The 24th ACM SIGKDD International Conference on Knowledge Discovery \& Data Mining, August 19--23, 2018, London, United Kingdom}
\acmPrice{15.00}
\acmDOI{10.1145/3219819.3219959}
\acmISBN{978-1-4503-5552-0/18/08}



\maketitle

\section{Introduction}

Approximately representing data by linear subspaces are widely applied in computer vision, such as, dynamic texture classification \cite{doretto2003dynamic}, video-based face recognition \cite{aggarwal2004system}, shape sequence analysis \cite{veeraraghavan2005matching}, action recognition \cite{data:MSRAct:slama2015accurate}, etc.
The popularity could be attributed to the excellence of the subspace in representing a datum consisting of numerous vectors rather than a single one.
For instance, a subspace can represent an action video by approximating the principal postures in bases \cite{data:MSRAct:slama2015accurate}.
For comparing such high-dimensional data, subspace representation is also efficient, due to the small size of subspace matrixes.
Moreover, the annoying minor perturbation in data can be mitigated by constructing subspaces \cite{turk1991eigenfaces,stewart1990perturbation}.

We are interested in subspace-based learning which is an approach to represent the data as a set of subspaces instead of vectors. 
In order to study the similarity/dissimilarity between subspaces, the learning problems are formulated on the Grassmann manifold, a set of fixed-dimensional subspaces embedded in a Euclidean space \cite{hamm2008grassmann}.
Unfortunately, because of its nonlinear structure, traditional learning techniques for Euclidean space are not directly feasible for the Grassmann manifold.
To tackle this issue, Grassmann kernels, e.g., the Binet-Cauchy kernel \cite{smola2005binet} and the Projection kernel \cite{edelman1998geometry}, are proposed and have achieved promising performance in discriminative learning \cite{hamm2008grassmann}. 
Particularly, the Projection kernel maps subspaces to projection matrixes lying in an isometric flattened Euclidean space.
Its Euclidean properties make traditional learning methods easily adapt to subspaces, e.g., discriminant analysis \cite{hamm2008grassmann,harandi2011graph}, dictionary learning \cite{harandi2013dictionary} and metric learning \cite{huang2015projection}.

Though discriminative learning on Grassmann manifolds has been widely studied, the stability of subspace representation is seldom discussed in the literature.
In implementations, subspaces are approximately constructed by Singular-Value Decomposition (SVD) of designed data matrixes or sets of vectors.
When noise appears in data, subspaces are prone to be rotated in unexpected directions.
Plus, the bases of subspace representation are practically selected by singular values, which may alter when signal amplitude fluctuates.
As a result, informative bases might be changed or missed from representation, if noise signals are overly significant.
For small intra-class differences between subspace instances, such disturbances will mislead classifiers to a great extent.

\begin{figure}[!ht]
    \centering
    \includegraphics[width=0.9\columnwidth]{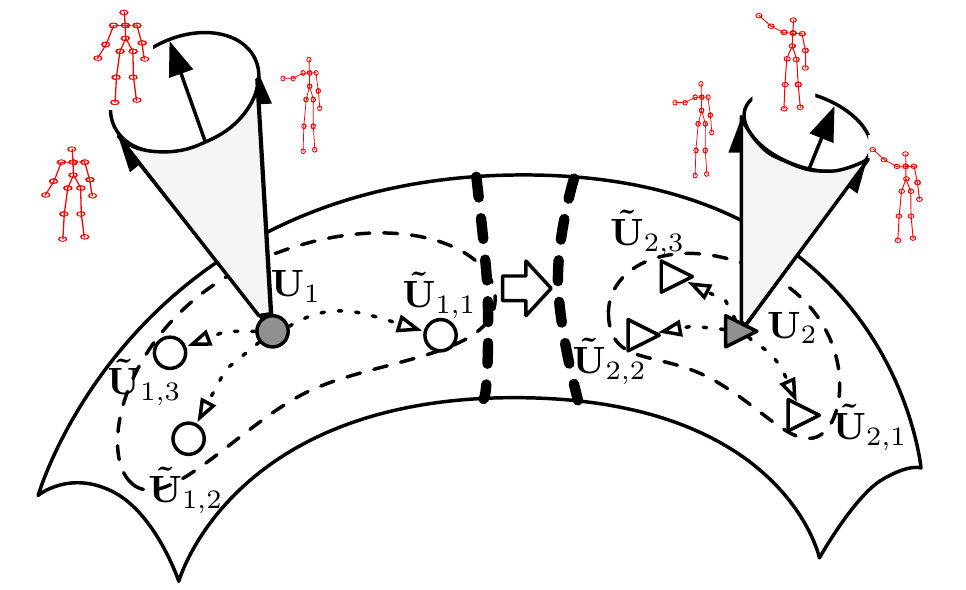}
    \caption{Illustration of learning with disturbances on the Grassmann manifold. Two gray points are subspace representations of action videos in which bases approximate the principal postures of an action. Disturbed subspaces (white points) are generated from them on the nonlinear manifold. The classification boundary (the thick dash line) is `pushed' away from $\mathbf{U}_1$ because the disturbance of $\mathbf{U}_1$ is prone to be more significant.}
    \label{fig:dg_manifold}
\end{figure}

Inspired by implicit data noising scheme \cite{maaten2013learning}, we investigate to take the disturbances into account in subspace-based discriminative learning.
We depict the disturbances on the manifold in \cref{fig:dg_manifold}.
For each instance, a large number of disturbed duplicates are generated subject to specific distributions on the nonlinear manifold and the learning objective function is modified by averaging over these duplicates.
Through dual optimization of the objective, the expectation approximation of the average yields a new kernel, i.e., Disturbance Grassmann kernel.

Specifically, two types of disturbance are discussed.
First, we consider a pseudo-Gaussian distribution on the Grassmann manifold.
The variance varies on bases since bases of larger singular values tend to be stabler than others with respect to noise.
Second, a Dirichlet distribution is introduced to characterize the probability of missing bases.
The scaled Projection kernels \cite{hamm2008grassmann} turns out to be a special case of the Dirichlet disturbance.
Furthermore, we analyze the noise modeled by subspaces, bridging the connection between subspace disturbances and data noise.
It indicates that our proposed kernels properly encode the prior knowledge of real noise.

We demonstrate the disturbance Grassmann kernels with the Support Vector Machines on action data.
The classifiers with proposed kernels suggest a better performance in comparison to state-of-the-art methods.
Also, advanced subspace-based discriminative algorithms are compared.
In challenging environments, e.g., recognition involving noise action or in low latency, the proposed kernels show their superior robustness.

\section{Related Work}
In these decades, there has been a large body of work concerning Grassmann kernels and learning with noise.
This section reviews related work in the context of the Projection kernel and implicitly noising, which are the main base of the DG kernel.

Thanks to the success of the Projection kernel in subspaces, researchers have paid attention to its extension to adapt to more complicated cases.
From the probabilistic interpretation of the kernel, Hamm and Lee extend it to affine subspaces and scaled ones \cite{hamm2008grassmann}, which yields new kernels.
Ways assembling multiple Projection kernels are proposed, to handle tensor data \cite{lui2010action,lui2012tangent} or data with multiple varying factors \cite{wang2016product}.
But these kernels do not exactly consider the disturbances of subspaces or the noise of data.
One related work is learning a low-rank representation on Grassmann manifolds \cite{wang2014low,wang2016product}, which is robust to minor non-informative disturbances.
However, it fails to take account of the specific characteristics of subspace disturbances caused by data, which could be pretty large.
Another try is perturbing data by affine transformations to help to explore the geometric structure of manifolds \cite{lui2008grassregist}.
It is accompanied by drawbacks that the artificial perturbation relies on specific data while the subspace representation is capable for a wider range of data, and the sampling of perturbation brings high computation burdens.
A better way is to consider a subspace-level noise and learn models by implicitly noising.

Different from methods which improves the quality of data \cite{wu2013novel}, noising scheme can be regarded as a kind of regularization \cite{wager2013dropoutreg} and has proven to be effective for training robust classifiers \cite{chen2014dropout,liu2014state} %
 or extracting better features \cite{chen2015marginalizing,li2016learningmcflabel}.
Maaten et al. firstly propose the implicitly noising scheme for general learning problems \cite{maaten2013learning}.
The method, learning with Marginalized Corrupted Features (MCF), utilizes the expectation to approximate the average of the objective function over corrupted features.
But it applies a uniform configuration of noise parameters for each instance neglecting the side information of data.
Researchers, therefore, exploit instance-specific distributions to improve the method \cite{zhuo2015adaptive,ye2017learningmah}.
The technique is also introduced in metric learning, for that noise may alter relationships of points \cite{qian2014distancedropout,ye2017learningmah}.
Nevertheless, their attempts focus on linear spaces.
Different from them, we leverage the idea of the linear MCF to improve Grassmann classifiers with the help of implicit disturbances on subspaces.
Rather than learning individual distributions from data, we determine the disturbance parameter according to the connection between subspace representation and data.
It is more efficient and simpler and encloses information of subspace representation meanwhile. %
Inspired by the model-based kernel for efficient time series classification \cite{chen2013model,chen2014learning}, we treat a subspace as a model of data, then implement noising scheme through extending the kernel of models, i.e., the Projection kernel.

\section{Preliminaries}
\label{sec:back}
Before presenting the Disturbance Grassmann kernel, we would like to give a brief summary of basic knowledge of the Grassmann manifold. %

A Grassmann manifold $\mathcal{G}(m, D)$ is a set of $m$-dimensional linear subspace of $\mathbb{R}^D$ for $0\le m \le D$ and is a compact Riemannian manifold of $m(D-m)$ dimension \cite{absil2009optimization}. 
An element of Grassmann manifolds can be expressed by a basis matrix $\mathbf{U} = [\mathbf{u}_1,\mathbf{u}_2,\dots,\mathbf{u}_m] \in \mathbb{R}^{D\times m}$ which spans a subspace as
$ \mathrm{span}(\mathbf{U}) = \{ \mathbf{x} | \mathbf{x}=\mathbf{U}\mathbf{v}, \forall \mathbf{v}\in \mathbb{R}^m \}. $
The basis matrix satisfies that $\mathbf{U}^T\mathbf{U} = \mathbf{I}_m$, where $\mathbf{I}_m$ is the identity matrix of size $m\times m$.
The equivalence relation between two elements, $\mathbf{U}$ and $\mathbf{U}'$, are defined such that there exists a $m \times m$ orthogonal matrix $\mathbf{Q}_m$ such that 
\begin{align}
\mathbf{U}' = \mathbf{U}\mathbf{Q}_m. \label{eq:equvi_grassmann}
\end{align}
As a result, the basis matrix is rotation-invariant to any $\mathbf{Q}_m$.
For discriminative learning, some distance measurements between elements on the Grassmann manifold have been discussed \cite{hamm2008grassmann}.
Here, we specifically consider the metric resulting from projection mapping, which embeds the Grassmann manifold into the space of symmetric matrices.

The projection mapping is defined as \cref{eq:proj_map} and the corresponding Projection kernel function is \cref{eq:projection_kernel}.
\begin{align}
\phi:~\mathcal{G}(m, D) \rightarrow \mathrm{Sym}(D),~\phi(\mathbf{U})=\mathbf{UU}^T \label{eq:proj_map} \\
\kappa_P(\mathbf{U}_1, \mathbf{U}_2) = \trace [(\mathbf{U}_1\mathbf{U}_1^T)(\mathbf{U}_2\mathbf{U}_2^T)] \label{eq:projection_kernel} %
\end{align}
where $\mathrm{Sym}(D)$ is a space consisting of of $D$-by-$D$ symmetric matrixes and $\trace(\cdot)$ denotes the trace operation.
Each point on the Grassmann manifold is uniquely mapped to one projection matrix \cite{helmke2007newton} and the notion of distances between them are well-preserved meanwhile \cite{chikuse2012statistics}.
The resultant projection metric is simply Euclidean in $\mathbb{R}^{D\times D}$ which approximate the true Grassmann geodesic distance up to a scale of $\sqrt{2}$ \cite{harandi2013dictionary}.
Thus it is easy for extending the Euclidean methods to the Grassmann manifold and gets popular in the research community.
For more information about the geometry of the Grassmann manifold, please refer to \cite{edelman1998geometry}.

\section{Disturbance Grassmann Kernels}

In this section, we first introduce learning problems with disturbed data, and then derive the Disturbance Grassmann kernels from the dual optimization of the learning objective function.
Based on the formulation, we extend the Projection kernel to subspaces with potential disturbances.

\subsection{Learning with Disturbed Subspaces}

Suppose a set of observation-label pair set $\mathcal{D} = \{(\mathbf{U}_n, y_n)\}_{n=1}^N$, where $\mathbf{U}$ is an observation on the Grassmann manifolds instead of data space.
Our goal is to predict class label $y$ given a new value of $\mathbf{U}$. %
Following the derivation of the MCF methods \cite{maaten2013learning}, we augment the data set with disturbances as $\mathcal{\tilde D} = \{\{(\mathbf{\tilde U}_{nk}, y_n)\}_{k=1}^K\}_{n=1}^N$ and average their loss over disturbed samples.
Each disturbed subspace, $\mathbf{\tilde U}_{nk}$, is drawn from a conditional distribution $P(\mathbf{\tilde U}| \mathbf{U}_{n})$ independently.
Let $\ell(\cdot)$ be the instance-specific loss function. 
The learning objective function can be written as a regularized empirical loss function, i.e.,
\begin{align}
L(\mathbf{w}) = C \sum_{n=1}^N {1\over K} \sum_{k=1}^K \ell(\mathbf{w}; (\mathbf{\tilde U}_{nk}, y_n)) + {1\over 2} \norm{\mathbf{w}}^2 \label{eq:primal_problem}
\end{align}
where $\mathbf{w}$ is the model parameter vector and $C$ is regularization parameter.
Consider $\ell$ to be logistic $(1+\exp(-y\mathbf{w}^T\phi(\mathbf{U})))$ or hinge loss $\max(0,1-y\mathbf{w}^T\phi(\mathbf{U}))$.
Minimizing the objective will lead to maximizing margins between classes.

Here, we do not optimize the objective function directly, but its dual problems instead.
There are some reasons for doing so.
First, the noising mechanism on the primal problems will be ineffective for nonlinear manifolds.
On the primal problems, data points are treated as vectors in the Euclidean space.
However, for nonlinear manifolds, points are not distributed uniformly in the embedded space.
Feature-based corruption may generate meaningless points away from the manifold or the corruption distribution is improper for the manifold.
For example, randomly altering entries of subspace matrixes would violate orthogonal constraint.

The second consideration is about the efficiency.
Subspace data are usually of high dimension, whose projection matrix will be larger as a result.
Computing the mapping in \cref{eq:proj_map} will be of $\mathcal{O}(mD^2)$ computational complexity.
In contrast, the kernel function in \cref{eq:projection_kernel} has as low computation cost as $\mathcal{O}(m^2D)$ if $\mathbf{U}_1^T \mathbf{U}_2$ is calculated at first\footnote{In many cases, the subspace dimension $m$ is much smaller than the data dimension $D$.}.
It is notable that the benefit comes up only when the number of training points is less than the data dimension.

\subsection{Dual Optimization and New Kernels}

By introducing a new variable, $\boldsymbol{\alpha}$, the dual optimization of \cref{eq:primal_problem} can be formulated as minimizing
\begin{align}
L^*(\boldsymbol{\alpha}) &= {1\over 2} \sum_{n,n'=1}^{N} {1\over K^2} \sum_{k,k'=1}^K \alpha_{nk}\alpha_{n'k'} \kappa(\mathbf{\tilde U}_{nk}, \mathbf{\tilde U}_{n'k'}) \notag \\
&\qquad - \sum_{n=1}^{N} {1\over K} \sum_{k=1}^K h(\alpha_{nk}) \label{eq:stat_dual_loss} \\
&\mathrm{s.t.} ~ g(\boldsymbol{\alpha}) = 0; ~ 0 \le \alpha_{nk} \le C ,~\forall n, k \notag
\end{align}
where $\kappa(\cdot)$ denotes a Grassmann kernel, $h(\alpha)$ the penalty term and  $g(\boldsymbol{\alpha}) = 0$ the equation constraint.
The penalty is $h(\alpha) = \alpha$ for SVMs, while $h(\alpha) = \alpha \log(\alpha) + (C-\alpha) \log(C - \alpha)$ for logistic regressions \cite{yu2011dual}.
The equation constraint is $\sum_{n=1}^{N} {1\over M} \sum_{k=1}^K y_{n} \alpha_{nk} = 0$ for SVMs and can be ignored for logistic regressions.

To avoid the number of coefficients $\alpha$ growing explosively when more disturbed samples are added, we let $\mathbf{\tilde U}_{nk}$ for any $k$ share the same coefficient $\alpha_n$.
Subsequently, the number of disturbed samples can be increased to infinity, and the objective function is, therefore, approximated by the expectation form as
\begin{align}
\mathcal{L}^*(\boldsymbol{\alpha}) &= {1\over 2} \sum_{n, n'=1}^{N} \alpha_{n}\alpha_{n'} \Ebb_{\mathbf{\tilde U}_{n},\mathbf{\tilde U}_{n'}}[ \kappa(\mathbf{\tilde U}_{n}, \mathbf{\tilde U}_{n'}) ] - \sum_{n=1}^{N} h(\alpha_{n}) \label{eq:dual_mcf} \\
&\mathrm{s.t.} ~ g(\boldsymbol{\alpha}) = 0; ~ 0 \le \alpha_{n} \le C ,~\forall n \notag
\end{align}
where we omit the conditional dependence of expectation on $\mathbf{U}_{n}, \mathbf{ U}_{n'}$ for brief.
It is remarkable that we distinguish $\mathbf{\tilde U}_{n}$ and $\mathbf{\tilde U}_{n'}$ as two different variables even if $n$ equals $n'$ when computing the expectation.
It can be proved from \cref{eq:stat_dual_loss} where they are averaged separately.

We notice the objective function, \cref{eq:dual_mcf}, is barely different from the non-noised version except for the kernel.
Hence, we extract the resultant kernel in \cref{eq:dual_mcf} as a new kernel, \emph{Disturbance Grassmann kernel}, i.e.,
\begin{align}
\kappa_{\mathrm{DG}}(\mathbf{U}_{n}, \mathbf{U}_{n'}) = \Ebb_{\mathbf{\tilde U}_{n},\mathbf{\tilde U}_{n'}|\mathbf{U}_{n}, \mathbf{U}_{n'}}[ \kappa(\mathbf{\tilde U}_{n}, \mathbf{\tilde U}_{n'}) ]. \label{eq:dg_kernel_original}
\end{align}
The tractability of \cref{eq:dg_kernel_original} depends on the form of the Grassmann kernel and the distribution.
With the DG kernels, each new subspace $\mathbf{U}$ is classified using SVM or logistic regressions.

\subsection{Extensions to the Projection Kernel}
Because of the simple form of the Projection kernel, the DG kernel is tractable and can be computed in an efficient way.
Assume disturbed subspaces subject to a distribution $P(\mathbf{\tilde U}| \mathbf{U}_{n})$ defined on the Grassmann manifold. %
Substitute the Projection kernel, \cref{eq:projection_kernel}, into the DG kernel formula in \cref{eq:dg_kernel_original}, approaching
\begin{align*}
\trace \left(\Ebb_{\mathbf{\tilde U}_{n}|\mathbf{U}_{n}}[ \mathbf{\tilde U}_n \mathbf{\tilde U}_n^T ] \Ebb_{\mathbf{\tilde U}_{n'}|\mathbf{U}_{n'}}[ \mathbf{\tilde U}_{n'} \mathbf{\tilde U}_{n'}^T ] \right). %
\end{align*}
Particularly, the definition also holds for $n'=n$, because $\mathbf{\tilde U}_{n'}$ and $\mathbf{\tilde U}_{n}$ are distinguished as two variables conditionally independent when $\mathbf{U}_{n}$ is determined. 
With the decomposition $\mathbf{\tilde U} \mathbf{\tilde U}^T = \sum_{l=1}^m \mathbf{\tilde u}_l \mathbf{\tilde u}_l^T$, the DG kernel extending the Projection kernel is given by
\begin{align}
\kappa_{\mathrm{DG}}(\mathbf{U}_{n}, \mathbf{U}_{n'}) &= \sum_{l=1}^m \sum_{s=1}^m \trace \left( \Ebb[ \mathbf{\tilde u}_l^n {\mathbf{\tilde u}_l^{nT}} ] \Ebb[ \mathbf{\tilde u}^{n'}_s \mathbf{\tilde u}_s^{n'T} ] \right). \label{eq:dg_kernel} %
\end{align}
In the next section, we will manifest that such expectation operation will not increase computation cost when proper distributions, $P(\mathbf{\tilde U}| \mathbf{U}_{n})$, are introduced.
For conciseness, we will refer the DG kernel as the extension of the Projection kernel, if not specified.

\section{Subspace Disturbance}
\label{sec:subspace_dist}

In practical implementations, a truncated subspace is used for representation rather than a full-rank one.
Assuming the full-rank subspace to be $\mathbf{U}_f = [\mathbf{u}_1, \cdots, \mathbf{u}_r]$ with normalized singular values\footnote{The normalization is practically meaningful for avoiding the influence of overall amplitude of data. For example, to compare two image sets, the difference of lightness is non-discriminative and can be eliminated by the normalization.} $\lambda_l$ for $l=1,\cdots,r$, the truncated version lying on $\mathcal{G}(m, D)$ can be written as 
\begin{align}
\mathbf{U} &= [\beta_1 \mathbf{u}_1, \cdots, \beta_r \mathbf{u}_r] \label{eq:U_beta} \\
\beta_l &= \mathbb{I}(\lambda_l > \lambda_m ), ~\forall l\in \{1,\cdots, r\} \label{eq:beta}
\end{align}
where $\mathbb{I}$ is the indicating function and $\lambda_m$ is the threshold singular value.
Specifically, we consider two kinds of disturbance:
\begin{enumerate}
    \item The disturbance of basis matrix $[\mathbf{u}_1, \dots, \mathbf{u}_r]$, which should follow a proper distribution on the Grassmann manifold;
    \item The fluctuation of $\lambda_l$, which leads to dropout noise of $\beta_l$.
\end{enumerate}
We will show the disturbance can imply analytical formula of the DG kernel.

\subsection{Pseudo-Gaussian Disturbance} %

There are two ways to define parametric Probability Density Functions (PDFs) on Grassmann manifolds.
One is to define models in the embedded high-dimensional Euclidean space.
However, the integral over a restricted nonlinear manifold is usually intractable.
For example, the Grassmann manifolds will lead to a spherical integral which has no simple solution.
The other is intrinsically restricting the models on the manifold without relying on any Euclidean embedding.
The general idea is to define a PDF in the tangent space of the manifold and `wrap' it back onto the manifold.
Because the tangent space is a vector space, it is possible to draw upon the pool of classical statistic methods.
For such potential, we pursue a model in this way.

Actually, people can define arbitrary PDFs in the tangent space and wrap it back to the Grassmann manifold via an exponential mapping \cite{edelman1998geometry}.
To make the PDF effective, it is necessary to restrict the domain of density functions such that the mapping is a bijection.
Previously, the idea has been applied for statistic estimation on the Grassmann manifold by truncating beyond a radius of $\pi$ \cite{turaga2011statistical}.

A straightforward thought is to extend the multivariate Gaussian distribution to tangent vectors.
The tangent vector of a Grassmann point $\mathbf{U}$ is
\begin{align}
\mathbf{H} = \mathbf{U}_{\perp} \mathbf{Z} %
\end{align}
where $\mathbf{Z}\in \mathbb{R}^{(D-m)\times m}$ and $\mathbf{U}_{\perp}$ is a $D$-by-$(D-m)$ matrix spanning the null space of $\mathbf{U}$.
The exponential mapping from the tangent space to the manifold is\footnote{The mapping is derived from the exponential geodesics on the Grassmann \cite{edelman1998geometry}, thus we use the name, exponential mapping, and denote the mapping as $\mathrm{Exp}(\cdot)$.}
\begin{align} 
\mathbf{\tilde U} = \mathrm{Exp}(\mathbf{H}) &= (\begin{array}{*{20}{c}}
  \mathbf{UV}_H & \mathbf{U}_H
\end{array}) \left(\begin{array}{*{20}{c}}
  \cos \boldsymbol{\Sigma}_H \\ 
  \sin \boldsymbol{\Sigma}_H
\end{array}\right) \mathbf{V}_H^T \label{eq:pi_U_H} \\
&\mathrm{s.t.}~\mathbf{H} = \mathbf{U}_H \boldsymbol{\Sigma}_H \mathbf{V}_H^T \label{eq:tangent_vec_svd}
\end{align}
where \cref{eq:tangent_vec_svd} is the compact SVD of the tangent vector.
Supposing the PDF of a truncated Gaussian distribution is $f(\mathbf{H})$, the expectation of a projection matrix is
\begin{align*}
\Ebb[\mathbf{\tilde U}\mathbf{\tilde U}^T] = \int \mathrm{Exp}(\mathbf{H}) f(\mathbf{H}) d\mathbf{H}
\end{align*}
which is unfortunately intractable because the matrix exponential mapping, \cref{eq:pi_U_H}, is not elementary.

Besides, the mapping has a fatal drawback that it cannot preserve the density as hoped.
A simple trial is wrapping the Gaussian distribution onto the sphere in $\mathbb{R}^3$.
Samples around the pole opposite to the mean point will be denser than expected because of the wrapping operation.
In addition, the truncated Gaussian distribution has no analytical form because the normalization factor cannot be expressed in terms of simple functions.

As far as the DG kernel is concerned, the expectation merely involves individual bases in \cref{eq:dg_kernel}.
Accordingly, it is unnecessary to apply a density explicitly.
Instead, we utilize an implicit density resembling a spherical Gaussian distribution on the Grassmann manifold which is uniform in all directions.
We observe that subspace matrixes can be equivalently expressed in a more natural way coordinated by direction vectors and radius in the tangent space. %
The conclusion is summarized in \cref{th:polar_coord} on the base of \cref{th:cos_sin_decompos,th:quiv_tangent_vec}. %

\begin{lemma}
\label{th:cos_sin_decompos}
(Tangent expression)
Given two subspace denoted by $\mathbf{U}$ and $\mathbf{\tilde U}$ on $ \mathcal{G}(m,D)$, there exists a $m$-by-$m$ orthogonal matrix $\mathbf{\tilde Q}_m$ such that 
\begin{align}
\mathbf{\tilde U}\mathbf{\tilde Q}_m &= \mathbf{U}\cos \boldsymbol{\Theta} + \mathbf{\hat H} \sin \boldsymbol{\Theta} \label{eq:cos_sin_decompos}
\end{align}
where $\boldsymbol{\Theta}$ is a $m$-by-$m$ positive diagonal matrix describing tangent distance in entries, and $\mathbf{\hat H}$ a $D$-by-$m$ tangent vector with orthonormal columns. %
\end{lemma}
\begin{lemma}
\label{th:quiv_tangent_vec}
Any tangent vector on the Grassmann manifold has an equivalent expression as $\mathbf{\hat H} \boldsymbol{\Theta}$. %
\end{lemma}
\begin{lemma}
\label{th:polar_coord}
Unitary tangent vector and distance can uniquely and compactly determine any subspace on Grassmann manifolds.
\end{lemma}
\begin{proof}
The first two lemmas can be proved from the exponential mapping in \cref{eq:pi_U_H,eq:tangent_vec_svd} with the equivalent relationship defined in \cref{eq:equvi_grassmann}.
The third lemma is naturally deducted from \cref{th:cos_sin_decompos,th:quiv_tangent_vec}.
The detailed proofs of lemmas in the paper can be found in supplementary materials\footnote{The supplementary file is available at \url{http://jyhong.com/files/sigkdd_supp.pdf}.}.
\end{proof}
Therefore, a basis of a randomly disturbed subspace can be formulated as
\begin{align}
\mathbf{\tilde u} = \mathbf{u} \cos \theta + \mathbf{w}\sin \theta \label{eq:basis_decompos}
\end{align}
where $\mathbf{w} = \mathbf{U}_{\perp} \mathbf{x}$ such that $\norm{\mathbf{x}}=1$, and $\mathbf{u}$ is a column of a subspace $\mathbf{U}$.
The $\theta$ is one entry of $\boldsymbol{\Theta}$. %
In correspondence, the $\mathbf{w}$ presents the direction.

With the tangent expression, we introduce a \emph{pseudo-Gaussian distribution} on the Grassmann manifold.
The direction vector is subject to a uniform direction distribution, denoted as $f(\mathbf{w})$.
And the $\theta$ subject to a proper distribution whose form is not specified though.
Without loss of generalization, we write it as $f(\theta;\sigma)$ which is governed by the variance $\sigma$.
In \cref{th:tg_exp_coef}, we show how the expectation can be expressed in a simple form with a coefficient.
\begin{lemma}
\label{th:tg_exp_coef}
Assuming a disturbed subspace $\mathbf{\tilde U}$ subject to the pseudo-Gaussian distribution on Grassmann manifold, then the expectation involving an individual basis will be given by
\begin{align}
\Ebb [\mathbf{\tilde u} \mathbf{\tilde u}^T] = c(\sigma) \mathbf{u}\mathbf{u}^T + {1-c(\sigma) \over D-m} \mathbf{U}_{\perp} \mathbf{U}_{\perp}^T \label{eq:tg_exp_uuT}
\end{align}
where $\sigma$ denotes the level of variance and $c(\sigma)$ is a function of $\sigma$ dominated between $0$ and $1$.
\end{lemma}
\begin{proof}
Substituting the integrals
\begin{align}
\int \mathbf{w}\mathbf{w}^T f(\mathbf{w}) d\mathbf{w} &= \int_{\mathbf{x}\in\mathbb{R}^{D-m},\norm{\mathbf{x}}=1} \mathbf{U}_{\perp} \mathbf{x} \mathbf{x}^T \mathbf{U}_{\perp}^T f(\mathbf{x}) d\mathbf{x} \notag  \\
&= {1\over D-m}\mathbf{U}_{\perp}\mathbf{U}_{\perp}^T \notag  \\
c(\sigma) &\triangleq  \int \cos^2 \theta f(\theta; \sigma) d\theta \label{eq:c_sigma}
\end{align}
and \cref{eq:basis_decompos} into the expectation
yields the expected result.
\end{proof}

The exact distribution on the Grassmann manifold and the expectation, \cref{eq:c_sigma}, are nontrivial for precise calculation.
However, a well-guess is that the coefficient is a monotonically decreasing differential function of variance, because the original basis should be less influential when the disturbance is more significant.
Finding an approximate function is not tough by fitting the boundary cases, the smallest and largest variance.
Because the form of \cref{eq:c_sigma} is not determined, it is free to define the range of $\sigma$.
We let $0$ and $1$ be the lower and upper bound, respectively.

\begin{figure}[!ht]
    \centering
    \includegraphics[width=\columnwidth]{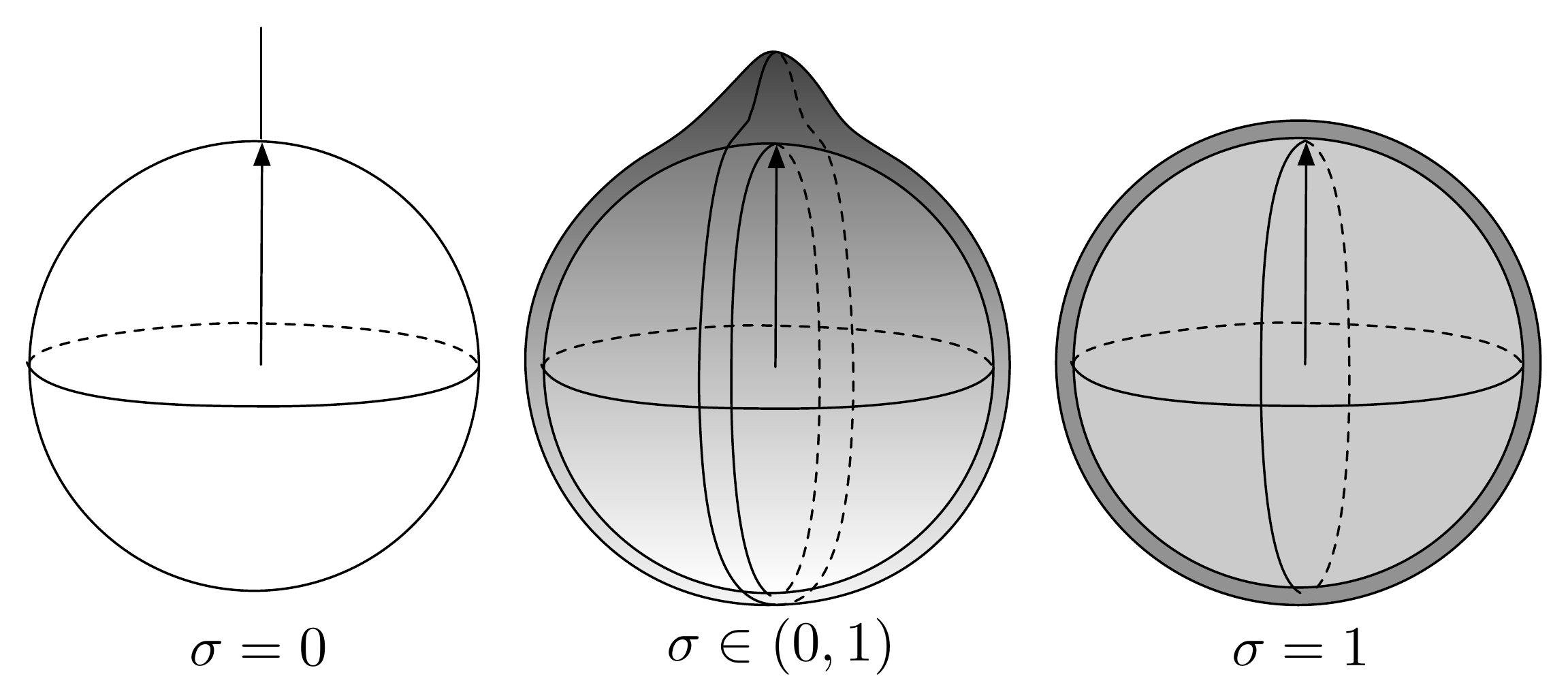}
    \caption{The pseudo-Gaussian distribution on a $2$-sphere for different variance $\sigma$. The arrow represents the mean basis. The darker color represents the higher probability density.}
    \label{fig:sphere_distribution}
\end{figure}

In \cref{fig:sphere_distribution}, we visualize the preferred distribution resembling a Gaussian distribution on a $2$-dimensional sphere.
When variance $\sigma$ is zero, the coefficient $c(\sigma)$ should be $1$.
While $\sigma$ is close to upper bound, the distribution should degenerate to an uniform distribution whose expectation should be
\begin{align*}
\Ebb [\mathbf{u} \mathbf{u}^T] = {1 \over D-m+1} \mathbf{u}\mathbf{u}^T + {1 \over D-m+1} \mathbf{U}_{\perp} \mathbf{U}_{\perp}^T.
\end{align*}
Hence, we can intuitively design the coefficient as
\begin{align*}
c(\sigma) = {1 \over \sigma^2 (D-m) + 1}.
\end{align*}
The following question is how the variance ought to be chosen.
Of course, a naive solution is uniformly disturbing all bases, whereas we recommend a singular-value-related variance
\begin{align}
\sigma^2 \propto {\epsilon^2 \over D}{1\over \lambda^2} \label{eq:variance_lambda}
\end{align}
governed by a parameter $\epsilon$.
From the practical standpoint, bases tend to be disturbed more easily when their singular values are small.
Theoretically, the variance expression rises up from the SVD perturbation theory under Gaussian noise \cite{wang2015singular} when the variance $\sigma$ is close to zero.
Because the approximation is a Euclidean assumption of the manifold corresponding to the tangent space and we have used precisely wrapping, the restriction for $\sigma$ is no longer essential.
To fit the variance in the range $[0,1]$, we let the square variance be a function of singular value as
\begin{align}
\sigma_\lambda^2 = 1 - \exp \left\{- {\epsilon \over D} \left({1\over \lambda} -1\right) \right\} \label{eq:sigma_lambda}
\end{align}
where $\lambda=1$ is correspond to the case when non-normalized singular value is approaching infinity.
Also the extremely small singular value $\lambda$ can suggest a boundary case of variance.
It is notable that in \cref{eq:sigma_lambda}, we remove the square operation in $\epsilon$ and $\lambda$ because it will be numerically easier for tuning $\epsilon$ without the square.

Let $\Delta = {m - \trace \Sigma \over D - m}$ and $\Sigma_{ii} = c(\sigma_{\lambda_i})$ and substitute the expectation in \cref{eq:tg_exp_uuT} into the DG kernel approaching 
\begin{align}
\kappa_{\mathrm{PG}}(\mathbf{U}_n, \mathbf{U}_{n'} ) &= \trace\left\{(\boldsymbol{\Sigma}_n-\Delta_n\mathbf{I}_m) \mathbf{U}_n^T \mathbf{U}_{n'} (\boldsymbol{\Sigma}_{n'}-\Delta_{n'}\mathbf{I}_m) \mathbf{U}_{n'}^T \mathbf{U}_{n}\right\} \notag \\
&\quad + \Delta_n \trace \boldsymbol{\Sigma}_{n'} + \Delta_{n'} \trace \boldsymbol{\Sigma}_n \notag \\
&\quad + \Delta_n \Delta_{n'} (D - 2m) \label{eq:PG_kernel}
\end{align}
which is named the DG Pseudo-Gaussian (DG-PG) kernel.
Its computation cost is up to $\mathcal{O}(m^2D)$.

\subsection{Dirichlet Disturbance}

We have assumed the singular values to be unchanged so far.
However, they will fluctuate when the amplitude of signal changes.
To incorporate such disturbances, we now temporarily let $\mathbf{u}_1, \dots, \mathbf{u}_r$ fixed.

The normalized singular values imply that $\sum_{l=1}^r \lambda_r = 1$.
Thus we stack the singular values in column vector $\boldsymbol{\lambda}$ and make the singular values following a Dirichlet distribution as
\begin{align*}
\boldsymbol{\tilde \lambda} \sim \mathrm{Dir}(\boldsymbol{\lambda})
\end{align*}
where the model parameter is set as the original singular values.

It is of our interest how probable a basis is retained in representation, or a singular value is larger than the threshold $\lambda_m$.
For the Dirichlet disturbance, we consider the marginal distribution of an individual singular value is a beta distribution. %
Its cumulative distribution function is, therefore, a regularized incomplete beta function $I_x(\lambda_l, 1 - \lambda_l)$.
Thus we can obtain
\begin{align}
P(\beta_l = 1) &= 1 - P(\lambda_l \le \lambda_m) \notag \\
&= I_{1 - \lambda_m}(1 - \lambda_l, \lambda_l) \triangleq p_l \label{eq:dropout_ra}
\end{align}
and the expectation of the projection matrix is equivalent to the projection matrix of a scaled subspace, i.e.,
\begin{align}
\Ebb[ \mathbf{\tilde U} \mathbf{\tilde U}^T ] = \sum_{l=1}^m \Ebb_{\beta_l}[ \beta_l \mathbf{u}_l \mathbf{u}_l^T ] = \sum_{l=1}^m p_l \mathbf{u}_l \mathbf{u}_l^T. \label{eq:proj_mat_dir}
\end{align}

Define a diagonal matrix $\boldsymbol{\Sigma}$ with $p_l$ in entries.
Then the disturbance kernel is a DG Dirichlet (DG-Dir) kernel, i.e.,
\begin{align}
\kappa_{\mathrm{Dir}}(\mathbf{U}_{n}, \mathbf{U}_{n'}) = \trace ( \boldsymbol{\Sigma}_{n} \mathbf{U}_{n}^T \mathbf{U}_{n'}\boldsymbol{\Sigma}_{n'}\mathbf{U}_{n'}^T \mathbf{U}_{n} ) \label{eq:dir_kernel}
\end{align}
whose time complexity is close to the original Projection kernel.

\subsection{Summary of Disturbance Grassmann Kernels}

In our derived DG kernels, the singular values play an important role in the coefficient of subspaces.
They can be written uniformly as
\begin{align*}
\kappa_{\mathrm{DG}}(\mathbf{U}_{n}, \mathbf{U}_{n'}) = \trace ( \mathbf{D}_{n} \mathbf{U}_{n}^T \mathbf{U}_{n'}\mathbf{D}_{n'}\mathbf{U}_{n'}^T \mathbf{U}_{n} ) + R
\end{align*}
where $\mathbf{D}$ denotes a positive diagonal matrix and $R$ is the residual term.
Interestingly, the scaled kernel proposed in \cite{hamm2009extended} has the identical form with $R=0$ and the entries of $\mathbf{D}$ are singular values.
They study the probabilistic manner of constructing subspace which is similar to the SVD and deduct the probability distance leading to extended Projection kernels.
Its success motivates us to think about the intrinsic statistic problems behind the Projection kernel.
Different from them, we emphasize the probability of disturbances in the SVD operation, which has more tight links to the data.

\section{Data Noise}

Until now, our discussions are centered on the noise of Grassmann points, lacking necessary links to data noise.
While in this section, we will analyze the origin of the discussed disturbances, to verify the firm connection between DG kernels and data noise.

\subsection{Singular-Vector Perturbation}

An asymptotic perturbation theory was previously proposed in \cite{bura2008distribution}.
It supposes that given any low-rank data matrix, as the perturbation converges to zero, a properly scaled version of its singular vectors corresponding to the nonzero singular values will converge in probability to a multivariate normal matrix.
Notably, there is no exact constraint on the type of minor noise in data.
As a consequence, approximated Gaussian noise in subspace matrixes can be a result of universal data noise, as long as it is fairly insignificant.

To reveal the form of approximated Gaussian noise, we start from the Gaussian perturbation in data.
Suppose $\mathbf{X}$ is a $D\times N$ data matrix, and $\mathbf{\tilde X} = \mathbf{X} + \epsilon_X \mathbf{W}$ is its noisy version controlled by a small constant, where the entries of $\mathbf{W}$ are independently and identically distributed.
For a given $m < D$, $\mathbf{X}$ has the singular-value decomposition $\mathbf{X} = \mathbf{U}\boldsymbol{\Sigma}\mathbf{V}^T$,
where $\boldsymbol{\Sigma}$ is a $m\times m$ singular-value diagonal matrix, and all singular values denoted by $\lambda_i$, $i=1,\dots, m$ are in descending order.
The noise leads to a non-asymptotic result with a closed-form of approximation error.
For the concern of this paper, we restate the results in Wang's paper \cite{wang2015singular} as \cref{th:gauss_noise}.

\begin{lemma}
\label{th:gauss_noise}
Given noised data matrix, i.e. $\mathbf{\tilde X} = \mathbf{X} + \epsilon_X \mathbf{W}$ and $W_{ij}\sim\mathcal{N}(0,1/D)$, the resultant subspace involves a Gaussian noise on the null space of $\mathbf{U}$ denoted by $\mathbf{U}_{\perp}$.
When $\epsilon_X$ approaches zero, the disturbed subspace can be written as
\begin{align}
\mathbf{\tilde U} = \mathbf{U} + \epsilon_X \mathbf{U}_{\perp} \mathbf{W}_0 \boldsymbol{\Sigma}^{-1}. \label{eq:U_gauss_perturbation}
\end{align}
where $\mathbf{W}_0$ is subject to the same Gaussian distribution.
\end{lemma}
\begin{proof}
The lemma can be proved according to the theorem of singular-vector perturbation under Gaussian noise \cite{wang2015singular}.
For brevity, we enclose the detailed proof in supplementary materials.
\end{proof}

Roughly speaking, to ensure the approximation error has the lowest degree, we let the $\epsilon$ lower than any of $D^{-\beta-1/2}$, $D^{-3/4}$, and $ \norm{\mathbf{U}}_{\mathrm{max}}^{-1} D^{-1}$, %
where $0 < \beta < 1/2$ controls the probability of variance of Gaussian approximation increasingly. 
In our implementation, we approximately let $\epsilon \le D^{-1/2}$.
Therefore, we get the relation between variance and the singular values in \cref{eq:variance_lambda}.

The above theories strengthen the connection between the pseudo-Gaussian disturbance and the data noise, especially the Gaussian noise.
Numerically, the analysis gives tips about the tendency of the variance taking account of the singular values.

\subsection{Variant Subspace Representations}
\label{ssec:vsr}
Given action videos, subspace representations vary mainly due to action variance.
We define these subspace representations as points $\mathbf{U}_1, \dots, \mathbf{U}_M$ on corresponding Grassmann manifolds $\mathcal{G}(m_1, D)$, $\dots$, $\mathcal{G}(m_M, D)$ where $m_1$ to $m_M$ are natural numbers smaller than the maximum order $m$.
These different subspaces reveal different variance of action videos, including adding and missing principal postures.
Naive selection of one subspace for representation will lead to a poor generalization of classifiers.

To avoid missing useful features or incorporating unwanted features, we propose learning classifiers with all possible representations on variant Grassmann manifolds.
Especially, these representations are obtained from training data, which do not require extra data.
We first learn a subspace of full rank as $\mathbf{U}$ on $\mathcal{G}(m, D)$.
For the learned subspace, a subset of bases is utilized.
Specifically we define $\boldsymbol{\beta}\in \{0,1\}^m$ where entry $0$ denotes the deletion of a basis.
Thus representation data could be augmented on variant Grassmann manifolds as
$ \{ \phi(\mathbf{U}_n; \mathcal{G}_{\mathbf{1}}) \}_{n=1}^N \rightarrow \{ \{ \phi(\mathbf{U}_n; \mathcal{G}_{\boldsymbol{\beta}_i}) \}_{k=1}^K \}_{n=1}^N $ %
where $K$ additive representations are assigned for each training instance.

Thus training with augmented subspaces is corresponding to dropout of bases with $\boldsymbol{\beta}$.
We note that from \cref{eq:proj_mat_dir}, the distribution can also be viewed as a special case of Bernoulli distribution.
It, therefore, corresponds to a dropout noise on bases where $0$ entry denotes the deletion of basis.
From this perspective, the Dirichlet disturbance is a reflection of learning with variant subspace representation.

\section{Experiments}

\begin{table*}[!t]
\centering
\caption{Summary of data sets.}
\small
\begin{tabular}{|p{1.7cm}|p{6cm}|p{4cm}|p{4cm}|}
\hline
Data set & Classes & Total number of actions & Experimental protocol \\
\hline \hline
KARD \cite{data:kard:gaglio2015human} & Ten gestures (horizontal arm wave, high arm wave, two-hand wave, high throw, draw X, draw tick, forward kick, side kick, bend, and hand clap) and eight actions (catch cap, toss paper, take umbrella, walk, phone call, drink, sit down, and stand up) & 10 subjects $\times$ 18 actions $\times$ 3 try = 540 actions & Experiment A: $1/3$ training/ $2/3$ testing per subject; Experiment B: $2/3$ training/ $1/3$ testing per subject; Experiment C: $1/2$ training/ $1/2$ testing. \\
\hline
UCFKinect \cite{data:UCFKinect:ellis2013exploring} & balance, climb up, climb ladder, duck, hop, vault, leap, run, kick, punch, twist left, twist right, step forward, step back, step left, step right & 16 subjects $\times$ 16 actions $\times$ 5 try = 1280 actions & 4-fold cross-validation \\
\hline
\end{tabular}
\label{tbl:act_dataset}
\end{table*}

\begin{table}[!h]
\centering
\caption{Parameter ranges for models.}
\begin{tabular}{|l|c|}
\hline
 Parameter & Range \\ \hline \hline
 $C$ & $\{10^{-4},10^{-3},\dots, 10^{5}\}$ \\ \hline
 $r$ & $\{1,2,\cdots,15\}$ \\ \hline
 $\epsilon$ & $\{10^{-6}, 10^{-2}, 0.05, 0.1, 0.2, \cdots, 1, 1.2, 1.7, 2, 5, 40\}$ \\ \hline
 $\lambda_m$ & $\{0.001, 0.01, 0.1, 0.3, 0.5, 0.7, 0.9\}$ \\ \hline
\end{tabular}
\label{tbl:param}
\end{table}

On real data in \cref{tbl:act_dataset}, we evaluate the proposed kernels, DG-PG kernels and DG-Dir kernels with the renowned Support Vector Machines and compare them to different kernels.
The parameter ranges for different models are listed in \cref{tbl:param}.
Cross-validation over the range of parameters is applied for parameter selection.

The first two parameters, $C$ and $r$ are the parameters for all subspace-based SVMs.
Typically, $C$ is the soft margin parameter of SVMs.
And, we tune the assumed uniform rank $r$ for all subspaces, since we have no idea what is the real rank of the data.
As a result, there will be two steps of truncation in \cref{eq:U_beta}, controlled by $\lambda_m$ and $r$.
For kernels other than the DG kernels, $\lambda_m$ (or $m$) and $r$ is same.

Selecting the parameter $\lambda_m$ of DG-Dir kernel is quite straight.
We scan from $0$ to $1$ for the best parameter.
In real experiments, we find that a sparse parameter set as showed in \cref{tbl:param} usually has a better performance with cross-validation.
It might be ascribed to that the size of data sets is not large enough.

As for DG-PG kernel, the parameter $\epsilon$ governing the noise is not simple.
Because the variance is a designed function of singular values, it has a non-trivial relation to the variance of the tangent noise.
For this reason, we empirically choose some parameters to make the curves of the function of singular value as distinct as possible.

As a comparison, we run SVMs with other Grassmann kernels, including the original Projection (Proj) kernel, the extended (scaled/affine/affine scale) Projection (ScProj/AffProj/AffScProj) kernels, \cite{hamm2008grassmann}, and the Binet-Cauchy (BC) kernel \cite{smola2005binet}.
Moreover, advanced discriminative algorithms on the Grassmann manifold, e.g., Projection Metric Learning (PML) \cite{huang2015projection}, Grassmann Discriminant Analysis (GDA) \cite{hamm2008grassmann}, Graph embedding Grassmann Discriminant Analysis (GGDA) \cite{harandi2011graph} are included.
These methods are typical dimensionality reduction methods on Grassmann manifolds.
Technically, PML and GGDA are exercised using public Matlab codes from the source papers.
The refined features are classified using SVMs.

\subsection{Activity Recognition}
We perform activity recognition experiments with different difficulties on Kinect Activity Recognition (KARD) data set, to assess classifiers fully.

\textbf{Data set} %
Here it is of our interest to use the provided $15$ joints of the skeleton in world coordinates.
The skeleton features contain $45$ continuous real values.

\begin{table}[!h]
\centering
\small
\caption{The KARD data set is divided into three subsets presenting increasing difficulty for recognition.}
\begin{tabular}{*{3}{|c}|}
\hline
Set 1 & Set 2 & Set 3 \\
\hline \hline
Horizontal arm wave    &  High arm wave     & Draw tick \\ \hline
Two-hand wave          &  Side kick         & Drink \\ \hline
Bend                   &  Catch cap         & Sit down \\ \hline
Phone call             &  Draw tick         & Phone call \\ \hline
Stand up               &  Hand clap         & Take umbrella \\ \hline
Forward kick           &  Forward kick      & Toss paper \\ \hline
Draw X                 &  Bend              & High throw \\ \hline
Walk                   &  Sit down          & Horizontal arm wave \\
\hline
\end{tabular}
\label{tbl:KARD_set}
\end{table}

\textbf{Setup} In \cite{data:kard:gaglio2015human}, Gaglio et al. proposed some evaluation experiments on the data set. 
We follow their configuration applying three different experiments (A/B/C). 
And according to the visual similarity between actions or gestures, these activities are separated into three sets of difficulties in ascending order.
As shown in \cref{tbl:KARD_set}, activity set 1 is the simplest one since it is composed of quite different activities while the other two sets include more similar actions and gestures.
And generalization ability of predictors is assessed according to classification error rates in test sets.

\textbf{Result}
We report the classification error rates in \cref{tbl:KARD_class_all}.
Each error rate is an averaged over ten random partitions of data set following different experimental settings.

\subsection{Appended Noise Activities}
Considering the heterogeneousness of action data, frames from unrelated source would lead to the activity videos represented by totally different subspaces.
We specify the case in our experiment by artificially incorporating noise action or gestures in activity data, named Appended Noise Activities (ANA) experiment.

\begin{figure}
\includegraphics[width=0.9\columnwidth]{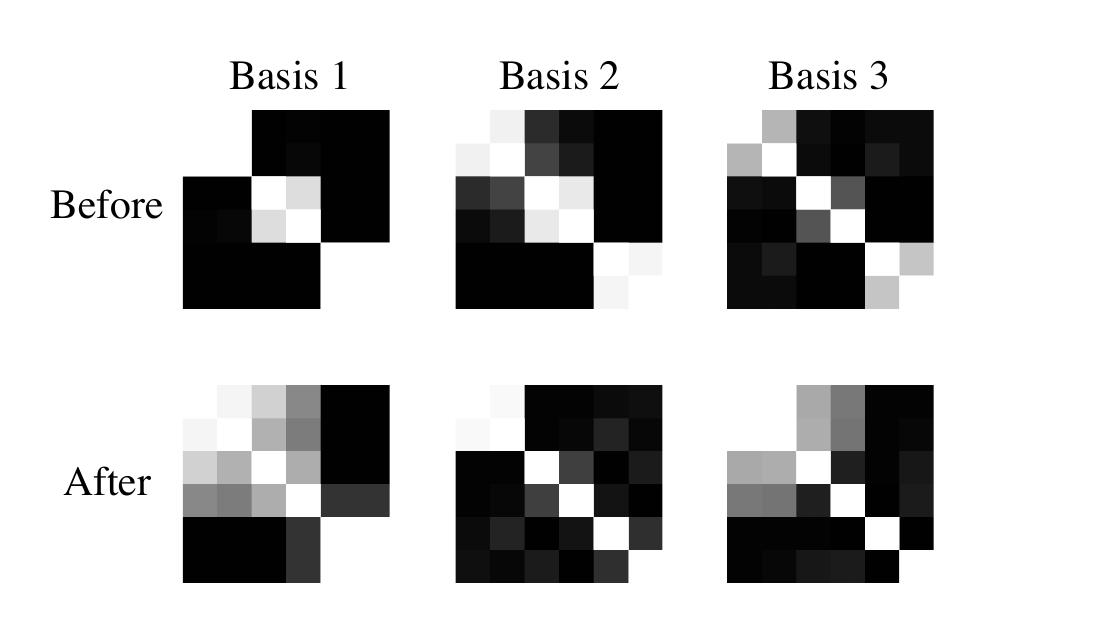} %
\caption{Similarity matrixes between actions measured by the Projection kernel of different bases before/after appending noise activities. Three class and two actions per class are tested. For each similarity matrix, instances are ordered by class and the darker color means the pair is less similar.  . %
}
\label{fig:ana_KARD_simmat}
\end{figure}

\begin{table*}[t]
\centering
\caption{Classification error rates on KARD data set.}
\small
\begin{tabular}{|l*{9}{|c}|}
\hline
EXP     & 1A            & 1B            & 1C            & 2A            & 2B            & 2C            & 3A            & 3B            & 3C            \\ \hline \hline
DG-PG  & 0.75          & 0.50          & \textbf{0.67} & 1.13          & 0.87          & 1.33          & 2.69          & 1.00          & 1.17          \\ \hline
DG-Dir  & \textbf{0.69} & \textbf{0.12} & 0.83          & \textbf{1.06} & \textbf{0.25} & \textbf{0.25} & \textbf{2.31} & \textbf{0.37} & \textbf{1.00} \\ \hline \hline
Proj     & 2.25          & 0.37          & 0.92          & 1.12          & 1.00          & 1.92          & 5.75          & 2.38          & 4.17          \\ \hline
ScProj & 2.50          & 0.75          & 1.67          & 1.31          & 0.50          & 1.08          & 4.62          & 1.00          & 2.75          \\ \hline
AffProj  & 7.69  & 2.00  & 4.67  & 3.44  & 1.00  & 3.25  & 7.25  & 3.37  & 6.00  \\ \hline
AffScProj &  9.50 &  4.62 &  6.50 &  5.50 &  3.75 &  5.50 &  9.75 &  5.12 &  8.50 \\ \hline
BC  & 6.25          & 3.50          & 4.83          & 5.50          & 2.88          & 3.92          & 9.88          & 4.00          & 6.33          \\ \hline
PML & 4.56  & 2.12  & 2.67  & 1.44  & 0.62  & 1.83  & 4.69  & 3.37  & 3.92  \\ \hline
GDA  &  9.44  &  4.13  &  7.00  &  3.81  &  1.75  &  2.42  & 15.25  &  7.25  & 11.42 \\ \hline
GGDA &  6.44 &  4.87 &  5.67 &  1.75 &  0.50 &  2.17 & 12.00 &  6.50 & 10.75 \\ \hline
\end{tabular}
\label{tbl:KARD_class_all}
\end{table*}

\begin{table*}[t]
\centering
\caption{Classification error rates on KARD data set with appended noise activities.}
\small
\begin{tabular}{|l*{9}{|c}|}
\hline
EXP     & 1A             & 1B            & 1C            & 2A            & 2B            & 2C            & 3A            & 3B            & 3C            \\ \hline \hline
DG-PG  & \textbf{11.25} & 8.25          & \textbf{9.08} & \textbf{4.62} & \textbf{1.62} & \textbf{2.50} & \textbf{7.06} & \textbf{3.87} & \textbf{4.92} \\ \hline
DG-Dir  & 12.56          & \textbf{7.38} & 9.92          & 5.31          & 1.88          & 3.08         & 10.12         & 4.38          & 5.83          \\ \hline \hline
Proj     & 12.06          & 9.75          & 9.58          & 7.44          & 3.00          & 4.00          & 12.13         & 6.62          & 9.33          \\ \hline
ScProj & 12.94          & 7.50          & 12.33         & 6.94          & 2.37          & 5.67          & 17.25         & 8.00          & 12.83         \\ \hline
AffProj  & 16.31 &  9.25 & 10.33 &  7.50 &  3.12 &  6.33 & 14.56 &  7.50 &  9.92 \\ \hline
AffScProj & 18.81 & 12.88 & 14.25 & 21.00 &  8.62 & 15.33 & 32.63 & 19.50 & 26.25 \\ \hline
BC  & 19.19          & 13.88         & 18.17         & 15.31         & 10.62         & 15.25         & 23.38         & 16.12         & 20.42         \\ \hline
PML  & 13.62  & 10.25  & 11.00  &  6.19  &  3.38  &  4.25  & 14.56  &  8.00  & 13.00  \\ \hline
GDA  &  13.94  &  10.38  &  10.50  &   8.06  &   2.50  &   4.42  &  21.12  &  11.12  &  15.58  \\ \hline
GGDA & 12.25   &  9.12   & 10.92   &  5.94   &  2.25   &  4.08   & 17.69   & 10.00   & 13.67 \\ \hline
\end{tabular}
\label{tbl:KARD:ana}
\end{table*}

\textbf{Setup}
For every experiment set defined in \cref{tbl:KARD_set}, we choose a noise class excluded from the set.
Then an activity sequence randomly picked from the noise class is appended to a sample sequence in the experiment set.
Such operation will significantly alter the singular vectors and values.
Because the noise action sequence is from the same class which is not distinguishable, the discriminative ability of each basis will be influenced meanwhile.
For illustration, we depict the similarity matrixes before and after adding noise activities in \cref{fig:ana_KARD_simmat}.
The three bases are discriminative as showed initially (the first row), while the first and third bases are not discriminative for class one and two after appending noise.

\textbf{Result}
\cref{tbl:KARD:ana} shows the classification error rates when noise activities are involved.
Still, the error rate is averaged over ten training/test partitions.

\subsection{Action Recognition in Low Latency}
In this experiment, we consider the challenge of action recognition in low latency, using the UCFKinect data set.
It is motivated by a daily requirement for action recognition, in which a rough estimation is required to be given before video recording is accomplished.
Except the technical necessity, a quick response of recognition system is in favor of better user experience.

\textbf{Data set} 
The skeletal joint locations over sequences of this data set are estimated using the Microsoft Kinect sensor and the PrimeSense NiTE.
The coordinates of $15$ joints are used for classification.
During recording, subjects are advised to start from and end at an identical resting posture for all actions.
The recording was manually stopped upon completion of the action.

\textbf{Setup }
The latency of an action is defined as the duration between the time a user begins the action and the time the classifier classifies the action.
To evaluate classifiers on different latency, the experiment in \cite{data:UCFKinect:ellis2013exploring} is introduced here.
The maximum video length $K$ is set to force all sequence truncated at the $K$-th frames.
Full videos are kept if shorter than $K$.

\begin{table}[!ht]
\centering
\small
\caption{Classification error rates on varying maximum video length.}
\begin{tabular}{|l*{7}{|c}|}
\hline
Frames  & 10             & 15             & 20             & 25            & 30            & 40            & 60            \\ \hline \hline
DG-PG  & 71.56          & \textbf{42.50} & 17.81          & \textbf{6.87} & \textbf{2.81} & 1.25          & \textbf{0.62} \\ \hline
DG-Dir  & \textbf{70.31} & 43.44          & \textbf{16.56} & \textbf{6.87} & \textbf{2.81} & \textbf{0.62} & \textbf{0.62} \\ \hline \hline
Proj     & 75.94          & 46.88          & 20.94          & 9.38          & 3.12          & 1.56          & 1.88          \\ \hline
ScProj & 73.12          & 45.63          & 17.19          & 7.81          & 3.12          & 1.56          & \textbf{0.62} \\ \hline
AffProj &  73.75 & 51.88 & 25.94 & 13.75 &  5.63 &  2.19 &  0.94  \\ \hline
AffScProj &  72.19 & 50.94 & 24.38 & 15.31 &  6.25 &  3.75 &  1.94  \\ \hline
BC  & 75.94          & 50.62          & 21.56          & 11.25         & 6.87          & 3.75          & 2.81          \\ \hline
PML & 73.44 & 48.12 & 19.06 & 10.31 &  3.44 &  1.88 &  1.25  \\ \hline
GDA & 80.94  & 56.88  & 22.50  & 13.44  &  6.87  &  2.19  &  2.81   \\ \hline
GGDA & 78.75  & 70.00  & 31.56  & 15.00  &  4.37  &  1.88  &  1.25  \\ \hline
LTB  & 78.13          & 50.63          & 25.63          & 13.13         & 7.92          & 2.71          & 2.09          \\ \hline \hline
LAL    & 86.09          & 63.05          & 35.23          & 18.44         & 9.45          & 4.84          & 4.06          \\ \hline 
\end{tabular}
\label{tbl:latency_UCFKinect}
\end{table}

\textbf{Result } We summarize the average classification error rate versus maximum video length in \cref{tbl:latency_UCFKinect}.
Some selected old results on the data sets are also reported in the table, where the Local Tangent Bundle (LTB) \cite{data:MSRAct:slama2015accurate} kernel had the best performance on subspaces before and the Latency Aware Learning (LAL) \cite{data:UCFKinect:ellis2013exploring} is the best non-Grassmann methods.
The lengths of frames to show the results are selected according to previous works \cite{data:UCFKinect:ellis2013exploring}.
When the length is longer than $30$, the difference of error rates will be less distinct.
Therefore the step of length is smaller after $30$.

\subsection{Discussion}
In the three data sets, our proposed kernels have shown superior performance compared to previous methods.
Specifically, DG-Dir kernels perform best in KARD data set and in low latency, while DG-PG kernels do better in the ANA experiment.

For different difficulties of the KARD sets, the error rates of DG-Dir kernels is stably lower than $3\%$ and better than those of the ScProj kernel, which use the fixed singular values as subspace coefficients.
We attribute the superiority to the more flexible relation between the coefficients and the singular values.

The low-latency challenge can be treated as a basis-missing problem for subspace-based learning.
As analyzed in \cref{ssec:vsr}, the case corresponds to a dropout noise of bases, which has been considered in the DG-Dir kernel.
The similarity between DG-Dir kernels and ScProj kernels is verified again in the experimental results.
The low-latency also has an impact on the singular vectors, which explains the good performance of DG-PG kernels.

The impact is more obvious in the ANA experiment.
Consequently, we observe a superiority of DG-PG kernels in the case.
The gaps between the DG-PG kernels and baseline methods are significant, especially in the set 2 and 3.

\section{Conclusion}
In this work, we novelly present the Disturbance Grassmann kernel concerning the potential disturbances for subspace-based learning.
By investigating the disturbances in subspaces, we extend the Projection kernel to two new kernels corresponding to the Gaussian-like noise in subspace matrixes and the Dirichlet noise in singular values.
We also prove the subspace disturbances are linked to data noise in specific cases.
With action data, we demonstrate that the resultant kernels can outperform the original Projection kernel as well as previously used extended Projection kernels, Binet-Cauchy kernels for subspace-based learning problems. %
Results with noisy environment also highlight the superiorities of the DG kernel.
The DG kernel is not confined to the Projection kernel and action data.
Extensions to other kernels has not been explored yet, and we look forward to a follow-up study in the future.
Besides, treating the subspace representation in model space, we expect to further exploit methods with metric co-learning \cite{chen2015model} and apply them to other sequential data \cite{li2016sequential}.

\begin{acks}
The authors would like to thank Prof. Peter Ti{\v{n}}o and Mr. Yang Li for their valuable and constructive suggestions during the planning of this research work.

This research is supported in part by the \grantsponsor{GS01}{National Key Research and Development Program of China}{} under Grant No.~\grantnum{GS01}{2017YFC0804003}, the \grantsponsor{GS02}{National Natural Science Foundation of China}{} under  Grant No.:~\grantnum{GS02}{91546116} and~\grantnum{GS02}{91746209}.
\end{acks}

\bibliographystyle{ACM-Reference-Format}
\balance
\bibliography{refs}

\end{document}